%% file: root.tex
\documentclass[letterpaper, 10 pt, conference]{ieeeconf}  

\IEEEoverridecommandlockouts                              

\input{commands}

\title{\LARGE \bf Reinforcement Learning with Temporal Logic Constraints for Partially-Observable Markov Decision Processes}

\author{Yu Wang, Alper Kamil Bozkurt, and Miroslav Pajic
\thanks{The authors are with the Pratt School of Engineering of Duke University, Durham, NC 27708, USA. Emails: {\tt\small \{yu.wang094, alper.bozkurt, miroslav.pajic\}@duke.edu}}%
\thanks{This work is sponsored in part by the ONR under agreements N00014-17-1-2504, N00014-20-1-2745 and N00014-18-1-2374, AFOSR award number FA9550-19-1-0169, and the NSF CNS-1652544 grant.}
}

\begin{document}

\maketitle
\thispagestyle{empty}
\pagestyle{empty}

\begin{abstract}
This paper proposes a reinforcement learning method for controller synthesis of autonomous systems in unknown and partially-observable environments with subjective time-dependent safety constraints. Mathematically, we model the system dynamics by a partially-observable Markov decision process (POMDP) with unknown transition/observation probabilities. The time-dependent safety constraint is captured by iLTL, a variation of linear temporal logic for state distributions. Our Reinforcement learning method first constructs the belief MDP of the POMDP, capturing the time evolution of estimated state distributions. Then, by building the product belief MDP of the belief MDP and the limiting deterministic B\"uchi automaton (LDBA) of the temporal logic constraint, we transform the time-dependent safety constraint on the POMDP into a state-dependent constraint on the product belief MDP. Finally, we learn the optimal policy by value iteration under the state-dependent constraint. 
\end{abstract}

\input{intro}

\input{prelim}

\input{setup}

\input{method}

\input{value}


\section{Conclusion}
\label{sec:conc}

This paper proposed a reinforcement learning method for controller synthesis of autonomous systems in unknown and partially-observable environments with subjective time-dependent safety constraints. We modeled the system dynamics by a partially-observable Markov decision process (POMDP) with unknown transition/observation probabilities and the time-dependent safety constraint by linear temporal logic formulas. Our Reinforcement learning method first constructed the belief MDP of the POMDP. Then, by building the product belief MDP of the belief MDP and the LDBA of the temporal logic constraint, we transformed the time-dependent safety constraint on the POMDP into a state-dependent constraint on the product belief MDP. Finally, we proposed a learning method for the optimal policy under the state-dependent constraint. 




\bibliography{ref}
\bibliographystyle{IEEEtran}

\end{document}

%% file: commands.tex
\let\labelindent\relax

\usepackage{amsmath}
\usepackage{amsthm}
\usepackage{amssymb}
\usepackage{enumitem} 
\usepackage{nicefrac}       
\usepackage{microtype}      
\usepackage{xcolor}
\usepackage[linkcolor=blue,hidelinks]{hyperref}

\usepackage[noadjust]{cite}

\usepackage{pifont}

\usepackage{url}

\usepackage{xspace}
\usepackage{booktabs} 
\usepackage{tikz}
\usetikzlibrary{automata}
\usetikzlibrary{positioning}
\usetikzlibrary{patterns}
\usetikzlibrary{shapes,arrows}
\tikzset{state/.style={circle, draw, minimum size=0.5cm}}

\usepackage{algorithm}
\usepackage{algorithmicx}
\usepackage{algpseudocode}

\usepackage{cleveref}

\newtheorem{definition}{Definition}
\newtheorem{theorem}{Theorem}
\newtheorem{remark}{Remark}

\newtheorem{lemma}{Lemma}

\newtheorem{problem}{Problem}
\newtheorem{corollary}{Corollary}

\crefname{section}{Section}{Sections}
\crefname{subsection}{Section}{Sections}
\crefname{definition}{Definition}{Definitions}
\crefname{problem}{Problem}{Problems}
\crefname{proposition}{Proposition}{Propositions}
\crefname{lemma}{Lemma}{Lemmas}
\crefname{theorem}{Theorem}{Theorems}
\crefname{corollary}{Corollary}{Corollaries}
\crefname{example}{Example}{Examples}
\crefname{figure}{Figure}{Figures}
\crefname{table}{Table}{Tables}
\crefname{assumption}{Assumption}{Assumptions}
\crefname{remark}{Remark}{Remarks}
\crefname{running}{Running Example}{Running Examples}
\crefname{algorithm}{Algorithm}{Algorithms}

\usepackage{xspace}
\usepackage{latexsym}
\usepackage{mathtools}
\usepackage{relsize}

\newcommand{\nat}{\mathbb{N}}

\newcommand{\real}{\mathbb{R}}

\newcommand{\argmax}{\mathrm{argmax}}

\newcommand{\dist}{\text{Dist}}
\newcommand{\Iff}{\text{ iff }}

\newcommand{\ex}{\mathbb{E}}

\newcommand{\G}{\Box}
\newcommand{\F}{\Diamond}
\newcommand{\X}{\bigcirc}
\newcommand{\U}{\mathbin{\mathcal{U}}}

\newcommand{\pomdp}{(S, A, T, p_0, r, \gamma, O, \Omega)}

\usepackage{todonotes}

%% file: intro.tex
\section{Introduction} \label{sec:intro}

Reinforcement learning methods are widely used to synthesize control policies for autonomous systems that work in unknown environments (e.g., robotics and unmanned vehicles)~\cite{sutton_ReinforcementLearningIntroduction_2018}. Among them, model-free reinforcement learning methods are of particular interest since they can derive control policies without identifying the complete system model~\cite{sutton_ReinforcementLearningIntroduction_2018}. Instead, they can find the best control policy for a given discounted reward function by iteratively rolling out a tentative control policy and improving it based on observed system behaviors. For systems with fully observable states, algorithms have been developed for various tasks~\cite{kober_ReinforcementLearningRobotics_2013,kober_LearningMotorPrimitives_2009,levine_EndtoEndTrainingDeep_2016} including linear temporal logic tasks~\cite{li_ReinforcementLearningTemporal_2017,hahn_OmegaRegularObjectivesModelFree_2019,hasanbeig_ReinforcementLearningTemporal_2019,bozkurt_ControlSynthesisLinear_2020,bozkurt_ModelFreeReinforcementLearning_2021}.

More often than not, the states of real-world autonomous systems such as self-driving cars~\cite{hubmann_AutomatedDrivingUncertain_2018} and robots~\cite{bozkurt_SecurePlanningStealthy_2021} are not fully observable due to system/environment uncertainty (e.g., sensor noise). For these applications, a widely-used mathematical model for control synthesis is the partially observable Markov decision process (POMDP)~\cite{kaelbling_PlanningActingPartially_1998}. A POMDP generalizes a Markov decision process (MDP), whose states are partially-observable through a probabilistic relation to a set of observations. Due to the partial observability, control synthesis is considerably harder for POMDP than fully observable models like Markov decision processes. Most of the exact decision problems on POMDP are either undecidable or PSPACE-complete~\cite{madani_UndecidabilityProbabilisticPlanning_1999,papadimitriou_ComplexityMarkovDecision_1987}.

This work studies the control synthesis for POMDPs in the Bayesian framework. Instead of exhaustively consider all states agreeing with the observation (e.g., in~\cite{madani_UndecidabilityProbabilisticPlanning_1999,papadimitriou_ComplexityMarkovDecision_1987}), we based the control actions on the posterior state estimation, i.e., beliefs. For a given observation, the evolution of beliefs under the control actions is captured by a Markov decision process with infinitely many states. Thus, the Bayesian framework ``lifts'' the POMDP control problem into a (fully-observable) MDP control problem at the cost of expanding the state space. Accordingly, we value iteration methods to deal with the infinite product belief space~\cite{pineau_PointbasedValueIteration_2003,ji_PointBasedPolicyIteration_2007}.

A crucial concern in learning-based control synthesis is dynamical safety. For safety-critical systems, such as self-driving cars~\cite{hubmann_AutomatedDrivingUncertain_2018} and robots~\cite{bozkurt_SecurePlanningStealthy_2021}, designing a policy that guarantees both safety and optimality is necessary. Typically, the safety constraints are time-dependent and expressible by linear temporal logic (LTL), a set of symbols and rules for formally representing and reasoning about time-dependent properties~\cite{pnueli_TemporalLogicPrograms_1977}. For different scenarios, variations of LTL in syntax and semantics are used~\cite{ahmadi_BarrierFunctionsMultiagentPOMDPs_2020}.

This work considers safety constraints expressed by iLTL, a variation of LTL for state distributions. It has found applications in wireless sensor network~\cite{kwon_LinearInequalityLTL_2004} and cyber-physical systems~\cite{wang_StatisticalVerificationDynamical_2015,wang_VerifyingContinuoustimeStochastic_2016}. Compare to barrier certificates~\cite{ames_ControlBarrierFunctions_2019}, iLTL is more expressive for time-dependent safety constraints. We use iLTL to capture safety constraints for posterior state estimations (i.e., beliefs) of the POMDP. Namely, when synthesizing the optimal control for a given discounted reward, if we ``believe'' an action will violate the iLTL safety constraint, we should not take it.

We propose a reinforcement learning method to derive the optimal control policy for a given discounted reward under an iLTL safety constraint. By constructing the belief MDP of the POMDP, we first lift the control synthesis problem to the belief space. Then, we build the product belief MDP of the belief MDP and the limiting deterministic B\"uchi automaton (LDBA) of the temporal logic constraint. This transforms the time-dependent iLTL constraint to a state-dependent B\"uchi constraint on the product belief MDP~\cite{bozkurt_ControlSynthesisLinear_2020}. Finally, we propose a value iteration method to learn the optimal policy for the discounted reward under the B\"uchi constraint. An overview of our approach is shown by \cref{fig:overview}.

The rest of the paper runs as follows. We give the definition of POMDP in \cref{sec:prelim} and formula the control synthesis problem \cref{sec:formulation}. We introduce the LDBA and build the product belief MDP in \cref{sec:product}. Then, we introduce the value iteration under constraints and the learning algorithm in \cref{sec:value iteration}. 
Finally, we conclude this work in \cref{sec:conc}.

\begin{figure}
\centering
\includegraphics[width=0.99\columnwidth]{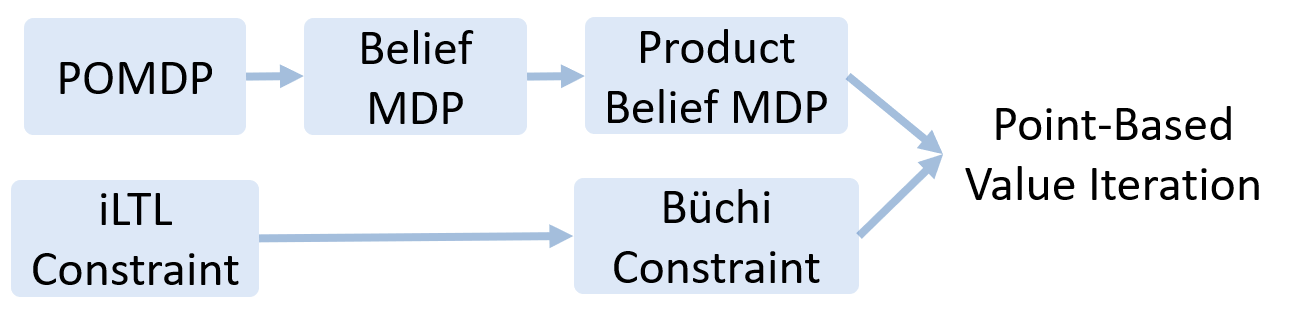}
\caption{An overview of our method.}
\label{fig:overview}
\end{figure}

%% file: prelim.tex
\section{Preliminaries} 
\label{sec:prelim}

We model an autonomous system's dynamics in the unknown and partially-observable environment by a partially-observable Markov decision process (POMDP), where the underline dynamics is a Markov decision process, but the control can only depend on observations probabilistically related to the states.

For a finite $S$, we denote the set of probability distributions on $S$ by $\dist (S)$. A POMDP is a tuple $\mathcal{M} = \pomdp$, where
\begin{itemize}
\item $S = \{s_1, \ldots, s_n\}$ is a set of $n$ states.
\item $A = \{a_1, \ldots, a_m\}$ is a set of $m$ actions. 
\item $T: S \times A \times S \to [0,1]$ is a set of transition probabilities between states,\footnote{For simplicity, we assume that all the actions are enabled for each state.} satisfying for any $s \in S$ and any $a \in A$,
\[
\sum_{s' \in S} T (s, a, s') = 1.
\]
By taking the action $a \in A$ on the state $s \in S$, the probability of transiting to the state $s' \in S$ is $T (s, a, s')$. 
\item $p_0 \in \dist(S)$ is an initial distribution.
\item $r: S \to \real$ is the immediate reward function.
\item $\gamma: S \to (0, 1)$ is a discount factor.
\item $O = \{o_1, \ldots, o_l\}$ is a set of $l$ observations.
\item $\Omega: S \times O$ is a set of observation probabilities,\footnote{The reward function $r$ and observation probabilities $\Omega$ depending on both the states and actions can fit into this setup by augmenting the MDP states with all its state-action pairs (e.g., as in~\cite{wang_HyperpropertiesRoboticsMotion_2020}).}
satisfying for any $s \in S$
\[
\sum_{o \in O} \Omega(s, o) = 1.    
\]
On the state $s$, the probability of yielding an observation $o$ is $\Omega(s, o)$. 
\end{itemize}
When $S = O$ and $\Omega(s, o) = 1$ if and only if $s = o$, the POMDP $\mathcal{M}$ reduces to a Markov decision process (MDP). We call a sequence of states $\sigma: \nat \to S$ a path of the POMDP if for any $t \in \nat$, there exists $a \in A$ such that $T(s(t), a, s(t+1)) > 0$. In addition, we call a sequences of state distributions $\sigma: \nat \to \dist(S)$ as an execution of the POMDP.

%% file: setup.tex
\section{Problem Formulation}
\label{sec:formulation}

We consider the control synthesis problem on the POMDP from \cref{sec:prelim}. Mathematically, the control policy 
\begin{equation} \label{eq:policy}
\pi: O^* \to \dist(A)
\end{equation} 
decides (probabilistically) the action to take from the history of observations. We denote the POMDP under the control policy by $\mathcal{M}_\pi$ and a random path drawn from the controlled POMDP by $\sigma \sim \mathcal{M}_\pi$. The control goal is to maximize the expected cumulative reward of the random path $\sigma$; equivalently, we look for the optimal control policy $\pi^*$ such that
\begin{equation} \label{eq:maximize}
\pi^* = \argmax_\pi \ex_{\sigma \sim \mathcal{M}_\pi} R (\sigma),
\end{equation} 
where     
\begin{equation} \label{eq:path reward}
R(\sigma) = \sum_{i=1}^\infty \gamma^t r(\sigma(t)).
\end{equation}

\subsection{Belief Markov Decision Processes}
\label{sub:belief mdp}

The states are not directly observable in a POMDP. But, from a history of observation, one can derive a probabilistic estimation of the states from the history of observations is called a belief $b \in \dist(S)$. At $t=0$, the belief $b_0$ is the initial distribution, i.e., $b_0 = p_0$. For $t>0$, the belief $b_t$ updates upon observing $o$ by the Bayes rule 
\begin{equation} \label{eq:belief update}
\begin{split}
b_{t+1} (s') = \frac{\Omega (s', o) \sum_{s \in S} T(s, a, s') b_t (s)}{\sum_{s' \in S} \Omega (s', o) \sum_{s \in S} T(s, a, s') b_t (s)}
\end{split}
\end{equation}

Following~\eqref{eq:belief update}, the POMDP $\mathcal{M}$ induces a Markov decision process (MDP) on the belief space $\dist (S)$ where each belief (i.e., state distribution) is a state.%
\footnote{The resulting belief MDP is a continuous state space, even if the ``originating'' POMDP has a finite number of states}
It shows how the belief changes for taking different actions. The belief MDP is fully observable.

\begin{definition} \label{def:belief mdp}
The belief MDP $\mathcal{D}$ of the POMDP $\mathcal{M} \allowbreak = (S, A, T, p_0, r, \gamma, O, \Omega)$ is defined by $\mathcal{D} = (\dist(S), \allowbreak A, T_\mathcal{D}, p_0, r_\mathcal{D}, \gamma)$ where
\[
T_\mathcal{D} (b, a, b') = \sum_{o \in O} 
\sum_{s' \in S} \sum_{s \in S} \eta(b, o, b') \Omega (s', o) T(s, a, s') 
b(s)
\]
with 
\[
\begin{split}
& \eta(b, o, b')
\\ & = \bigg\{
\begin{array}{ll}
1, & \text {if the belief update for } b, o \text { by~\eqref{eq:belief update} returns } b', 
\\ 0, & \text {otherwise,}\end{array}            
\end{split}
\] 
and
\[
r_\mathcal{D} (b, a)=\sum_{s \in S} b(s) r(s, a)
\]
for $b, b' \in \dist (S)$ and $a \in A$.
\end{definition}

\subsection{Time-Dependent Safety Constraints}
\label{sub:safety}

We formally capture time-dependent safety constraints by temporal logic. For the probabilistic safety related to state distributions, we introduce a variant of the linear temporal logic (LTL) to capture time-dependent specifications for wireless sensor network~\cite{kwon_LinearInequalityLTL_2004} and cyber-physical systems~\cite{wang_StatisticalVerificationDynamical_2015,wang_VerifyingContinuoustimeStochastic_2016}.

An inequality LTL (iLTL) formula is derived recursively from the rules
\begin{equation} \label{def:iltl syntax}
    \varphi \Coloneqq f 
    \mid \neg \varphi
    \mid \varphi \land \varphi 
    \mid \X \varphi 
    \mid \varphi \U_T \varphi 
\end{equation}
where 
\begin{itemize}
    \item $f: \real^n \to \real$ is a (given) function and is called an atomic proposition.
    \item $T \in \nat$ and $\X$ and $\U$ are temporal operators, 
    meaning ``next'' and ``until'', respectively.
\end{itemize}
Other common logic operators can be derived as follows: 
$\text{True} \equiv \varphi \land \neg \varphi$,
$\varphi \land \varphi' \equiv \neg(\neg \varphi \wedge \neg \varphi')$, 
$\varphi \rightarrow \varphi' \equiv \neg \varphi \land \varphi'$,
$\F_T \varphi \equiv \text{True} U_T \varphi$, 
and $\G_T \varphi \equiv \neg \F_T \neg \varphi$. 
Also, we denote $\U_T, \F_T,$ and $\G_T$ by $\U, \F$ and $\G$, respectively.
Our definition of iLTL is more general than~\cite{kwon_LinearInequalityLTL_2004,wang_StatisticalVerificationDynamical_2015,wang_VerifyingContinuoustimeStochastic_2016}, since we allow atomic proposition $f$ in~\eqref{def:iltl syntax} to be nonlinear functions.

Let $\sigma: \nat \to \dist (S)$ be a execution of the POMDP. The satisfaction (denoted by $\models$) of an iLTL formula is defined recursively by the rules
\begin{align*}
& \sigma \models f \Iff f(\sigma(0)) > 0 
\\ & \sigma \models \neg \varphi \Iff \sigma \not\models \varphi
\\ & \sigma \models \varphi_1 \land \varphi_2 \Iff \sigma \models \psi_1 \text{ and } \sigma \models \varphi_2
\\ & \sigma \models \X \varphi \Iff \sigma^{[1]} \models \varphi
\\ & \sigma \models \varphi_1 \U_T \varphi_2 \Iff \exists t \leq T. \big( \sigma^{[t]} \models \varphi \text{ and } \forall t' \leq t. \ \sigma^{[t']} \models \varphi \big)
\end{align*}
where $\sigma^{[t]}$ is the $t$-shift of $\sigma$ defined by $\sigma^{[t]}(t') = \sigma(t + t')$ for any $t' \in \nat$.

Using iLTL, we can express a wide class of time-dependent safety properties. For example, consider a general probabilistic space $S$. Let $S_1, S_2 \subseteq S$ and $I_{S_1}$ and $I_{S_2}$ be the indicator functions%
\footnote{That is, $I_{S_1} (s) = 1$ for $s \in S_1$ and $I_{S_1} (s) = 0$ otherwise.}, respectively. Then the iLTL formula
\begin{equation} \label{eq:ex1}
\varphi = \F \G (0.1 - I_{S_1}) \lor \F \G (I_{S_2} - 0.8)
\end{equation}
means that the probability of finally always staying in the (unsafe) set $S_1$ should be less than $0.1$ and the probability of finally always staying the (goal) set $S_2$ should be great than $0.8$.

\subsection{Control Synthesis Under iLTL Constraints}
\label{sub:formulation}

Based on modeling the system by a POMDP and the time-dependent safety constraint by iLTL, we formally introduce the problem formulation below.

\begin{problem} \label{problem}
Given the MDP from \cref{sec:prelim} and the iLTL constraint $\varphi$ from \cref{sub:safety}, find a policy $\pi$ in the form of~\eqref{eq:policy} maximize the expected value of the discounted reward~\eqref{eq:path reward}, while ensuring the belief sequence $b_0 b_1 \ldots$ derived from~\eqref{eq:belief update} satisfies $\varphi$.
\end{problem}

In \cref{problem}, the constraint $b_0 b_1 \ldots \models \varphi$ means when synthesizing the optimal control for a given discounted reward, if we ``believe'' an action will violate the iLTL safety constraint, we should not take it. Besides, in our Bayesian framework discussed in \cref{sub:belief mdp}, the control policy $\pi: O^* \to \dist(A)$ depends on past observations through belief updates (which are a sufficient statistic). Thus, to maximize the discounted reward~\eqref{eq:path reward} (without the safety constraint $\varphi$), we only need the belief sequence $b_0 b_1 \ldots$. In addition, since the satisfaction of the safety constraint $\varphi$ also involves the belief sequence $b_0 b_1 \ldots$, it suffices to consider a policy mapping belief sequences to actions. This is summarized by the following lemma.

\begin{lemma} \label{lem:policy}
To solve \cref{problem}, it suffices to find a policy
\begin{equation} \label{eq:policy2}
\pi: \dist (S)^* \to A.
\end{equation} 
\end{lemma}

%% file: method.tex
\section{Product Belief MDP}
\label{sec:product}

Following \cref{lem:policy}, the control policy that solves \cref{problem} depends on the past belief sequence, so is memory-dependent; this is beyond the capability of reinforcement learning. To remove the memory dependency, we generalize the product technique from~\cite{bozkurt_ControlSynthesisLinear_2020} to belief MDPs that have infinite states.

\subsection{Limiting-Deterministic B\"uchi Automata}

An LDBA is a tuple $\mathcal{A} = (Q, \Sigma, \delta, q_0, B)$ where
\begin{itemize}
\item $Q$ is a finite set of states;
\item $\Sigma$ is a finite set of alphabets;
\item $\delta: Q \times (\Sigma \cup \{\varepsilon\}) \to Q$ is a (partial) transition function (i.e., all alphabets are allowed on each state) with $\varepsilon$ standing for the empty alphabet. 
\item $q_0 \in Q$ is an initial state;
\item $B$ is a set of accepting states.
\end{itemize}
The LDBA satisfies that 
\begin{itemize}
\item the transition $\delta$ is total except for the empty alphabet, i.e., $|\delta(q, \cdot)| = 1$ for any $q \in Q$;
\item there exists a bipartition of into an initial and accepting component, i.e., $Q = Q_A \cup Q_I$ such that
\begin{itemize}
\item transitions from the accepting component stay within it, i.e., $\delta(q, \cdot) \subseteq Q_A$ for any $q \in Q_A$;
\item the accepting states are in the accepting component, i.e., $B \subseteq Q_A$.
\item the $\varepsilon$-moves are not allowed in the accepting component, i.e., $\delta(q, \varepsilon) = \emptyset$ for any $q \in Q_A$;
\end{itemize}
\end{itemize}

We call $q: \nat \to Q$ a path of the LDBA if $q(0) = q_0$ and for any $t \in \nat$, there exists $\sigma \in \Sigma$ such that $q(t+1) = \delta (q(t), \sigma)$. The path $q$ is accepted if $\text{Inf} (q) \cup B \neq \emptyset$, i.e., some state in $B$ appears infinitely often in $q$. Accordingly, an alphabet sequence is accepted if there exists a corresponding accepted path. As a variation of the standard linear temporal logic (LTL), an iLTL formula yields a graphic representation called a limiting-deterministic B\"uchi automaton (LDBA)~\cite{sickert_LimitDeterministicBuchiAutomata_2016,bozkurt_ControlSynthesisLinear_2020}.

\begin{lemma} \label{lem:ldba}
For any iLTL formula $\varphi$, there exists an LDBA $\mathcal{A}$ (whose alphabets are sets of iLTL atomic propositions from~\eqref{def:iltl syntax}) such that a sequence $b \models \varphi$ if and only if $b$ is accepted by $\mathcal{A}$. Accordingly, we call $\mathcal{A}$ realizes $\varphi$.
\end{lemma}

For example, an LDBA realizing the iLTL formula~\eqref{eq:ex1} is shown by \cref{fig:ex1}.

\begin{figure}
\centering
\includegraphics[width=0.7\columnwidth]{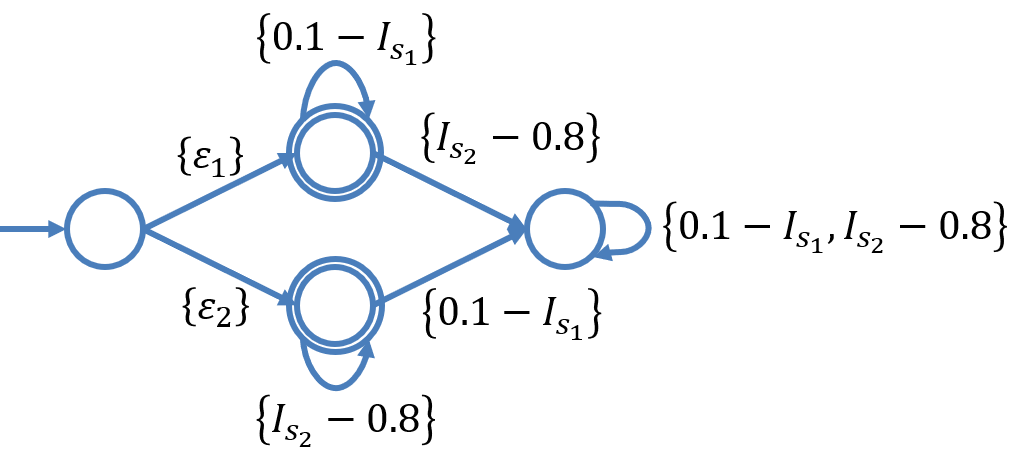}
\caption{The figure shows an LDBA for the iLTL formula~\eqref{eq:ex1}. We index the empty alphabets $\varepsilon_1$ and $\varepsilon_2$ to distinguish. The accepting states are double circled.}
\label{fig:ex1}
\end{figure}

\subsection{Product Belief MDP} 
\label{sub:product belief mdp}

\begin{definition} \label{def:product}
The product belief MDP $\mathcal{D}^\times = \mathcal{D} \times \mathcal{A}$ of the belief MDP $\mathcal{D} = (\dist(S), A, T_\mathcal{D}, p_0, \allowbreak r_\mathcal{D}, \gamma)$ and an LDBA $\mathcal{A} = (Q, \Sigma, \delta, q_0, B)$ is defined by $\mathcal{D}^\times = (\dist(S) \times Q, A \cup \{\varepsilon\}, T^\times, (p_0, q_0), r^\times, \gamma)$ where
\begin{align}
& T^\times \big( (b, q), a, (b', q') \big) =
\notag \\ & \qquad
\begin{cases}
    T_\mathcal{D} (b, a, b') & \text{if } a \in A,  q' = \delta(q, L(b)) \\
    1 & \text{if } a = \varepsilon, q' \in \delta(q, \varepsilon), b = b' \\
    0 & \text{otherwise}
\end{cases} \label{eq:T times}
\\ & L(b) = \{ f \text{ is an atomic proposition in } \varphi \mid f (b) > 0 \} \label{eq:labels}
\end{align}
and
\begin{equation} \label{eq:product reward}
r^\times \big( (b, q) \big) = r_\mathcal{D} (b).    
\end{equation}
In addition, let 
\begin{equation} \label{eq:b states}
    B^\times = \{ (b, q) \in \dist(S) \times Q \mid q \in B \}.
\end{equation}
\end{definition}

A path $(b_0, q_0) (b_1, q_1) \ldots$ of the product belief MDP corresponds uniquely to the combination of a path $b_0 b_1 \ldots$ of the belief MDP and a path $q_0 q_1 \ldots$ of the LDBA; and vice versa. Following~\eqref{eq:b states}, the path $q_0 q_1 \ldots$ is accepted by the LDBA (i.e., some states in $B$ appears in it infinitely often) if and only if some states in $B^\times$ appear infinitely often in the product path $(b_0, q_0) (b_1, q_1) \ldots$. Also, the reward~\eqref{eq:product reward} for the path $(b_0, q_0) (b_1, q_1) \ldots$ is equal to the reward~\eqref{eq:path reward} for the path $(b_0, q_0) (b_1, q_1) \ldots$. The existence of $\varepsilon$ moves in the LDBA (and the $\varepsilon$ actions in the product belief MDP) does not affect the path correspondence. Thus, we can reduce the iLTL constraint on the belief MDP to a B\"uchi constraint (i.e., visiting certain states infinitely often) on the product belief MDP, as stated by the following lemma.

\begin{theorem} \label{thm:convert}
If $\mathcal{D}^\times = \mathcal{D} \times \mathcal{A}$, then
\begin{equation} \label{eq:convert}
\max_{\pi} \mathbb{P}_{\sigma \sim \mathcal{D}_{\pi}} (\sigma \models \varphi) = \max_{\pi^\times} \mathbb{P}_{\sigma^\times \sim \mathcal{D}^\times_{\pi^\times}} (\sigma^\times \models \G \F B^\times \big)
\end{equation}
and
\begin{align} \label{eq:convert2}
& \max_{\pi} \mathbb{E}_{\sigma \sim \mathcal{D}_{\pi}} ( \sum_{i=1}^\infty \gamma^t r(\sigma(t)) ) 
\notag \\ & \qquad = \max_{\pi^\times} \mathbb{E}_{\sigma^\times \sim \mathcal{D}^\times_{\pi^\times}} ( \sum_{i=1}^\infty \gamma^t r^\times (\sigma^\times (t)) ),
\end{align}
where $\pi$ is of the form~\eqref{eq:policy2}; and similarly $\pi^\times$.  
\end{theorem}

\begin{proof}
Follow from the proof Theorem 1 in~\cite{bozkurt_ControlSynthesisLinear_2020}, noting that $L(b)$ in~\eqref{eq:labels} contains all atomic propositions holds on the (belief) state $b$.
\end{proof}

Furthermore, the maxima of the right-hand side of~\eqref{eq:convert} and~\eqref{eq:convert2} can be memoryless. Despite the existence of $\varepsilon$ actions, this memoryless policy maps to a memory-dependent policy on the belief MDP, which maximizes the left-hand side of~\eqref{eq:convert} and~\eqref{eq:convert2}, as formally stated below.

\begin{corollary} \label{cor:convert}
A memoryless policy maximizes the right-hand side of~\eqref{eq:convert} and~\eqref{eq:convert2}. It induces a memory-dependent policy that maximizes the left-hand side of~\eqref{eq:convert} and~\eqref{eq:convert2}.
\end{corollary}

\begin{proof}
Follow from the proof Theorem 3 in~\cite{sickert_LimitDeterministicBuchiAutomata_2016}.
\end{proof}

%% file: value.tex
\section{Learning under Constraints}
\label{sec:value iteration}

The product belief MDP has an infinite states space $\dist (S) \times Q$. Thus the tabular reinforcement learning method does not apply. Here, we generalize the value iteration method~\cite{smallwood_OptimalControlPartially_1973,pineau_PointbasedValueIteration_2003} to solve for the optimal control policy under the constraint.

\subsection{Bellman Equation on Belief Space}

We define the value function for the reward $r^\times$ on the product belief space by
\begin{equation} \label{eq:VR}
V_r \big( (b, q) \big) = \max_{\pi^\times} 
\mathbb{E}_{\sigma^\times \sim \mathcal{D}^\times_{\pi^\times, (b, q)}} 
\sum_{t \in \nat} \gamma^t r^\times (\sigma(t)) 
\end{equation}
where $\sigma^\times \sim \mathcal{D}^\times_{\pi^\times, (b, q)}$ is a random path drawn from $\mathcal{D}^\times$ under the policy $\pi$ from the product state $(b, q)$. The value function $V_r$ captures the maximal expected value of the reward~\eqref{eq:path reward} if started from the product state $(b, q)$. By~\cref{thm:convert}, $V_r (p_0, q_0)$ is the maximal expected reward of~\eqref{eq:path reward} on the POMDP $\mathcal{M}$.   

In addition, we define the value function for the B\"uchi constraint on the product belief space by
\begin{equation} \label{eq:VP}
V_p \big( (b, q) \big) = \max_{\pi^\times} \mathbb{P}_{\sigma \sim \mathcal{D}^\times_{\pi^\times, (b, q)}} 
\big( \sigma \models \G \F B^\times \big )
\end{equation}
The value function $V_p$ captures the maximal satisfaction probability of the B\"uchi constraint $\G \F B^\times$ if started from the product state $(b, q)$. By~\cref{thm:convert}, $V_p (p_0, q_0)$ is the maximal satisfaction probability of the iLTL safety constraint $\varphi$ on the POMDP $\mathcal{M}$.

From~\cref{cor:convert}, it suffices to consider pure and memoryless policy to maximize the two wo values functions~\eqref{eq:VR} and~\eqref{eq:VP}. Accordingly, they satisfy the following Bellman equations
\begin{align} 
& V_r \big( (b, q) \big) = \max_{a \in A} Q_r \big( (b, q), a \big) \label{eq:BellmanR_0}
\\ & Q_r \big( (b, q), a \big) = r^\times \big( (b, q) \big) + \gamma \sum_{q' \in Q} \int_{b' \in \dist(S)}
\notag \\ & \qquad T^\times \Big( \big( (b, q) \big), a, \big( (b', q') \big) \Big) V_r \big( (b', q') \big) \mathrm{d} b' \label{eq:BellmanR}
\end{align}
and 
\begin{align} \label{eq:BellmanP}
& V_p \big( (b, q) \big) = \max_{a \in A} Q_p \big( (b, q), a \big) 
\notag \\ & Q_p \big( (b, q), a \big) = \sum_{q' \in Q} \int_{b' \in \dist(S)} 
\notag \\ & \qquad  T^\times \Big( \big( (b, q) \big), a, \big( (b', q') \big) \Big) V_p \big( (b', q') \big) \mathrm{d} b'
\end{align}
For solving~\eqref{eq:BellmanP}, we also have 
\begin{align} \label{eq:BellmanP_2}
V_p \big( (b, q) \big) = 1 \text{ for } (b, q) \in B^\times
\end{align}
where $B^\times$ is given by~\eqref{eq:b states}, according to iLTL syntax.

\subsection{Value Iteration}

Our learning method aims to find solutions to the Bellman equations~\eqref{eq:BellmanR} and~\eqref{eq:BellmanP} by sampling, without using knowledge on the transition probabilities $T^\times$. Since the product belief space $\dist(S) \times Q$ is infinite, tabular learning methods are not applicable. Instead, we generalize the value iteration method~\cite{smallwood_OptimalControlPartially_1973} to the product belief space. Our method is based on the fact that the value functions and Q-functions are piecewise linear, as stated below.

\begin{lemma} \label{lem:piecewise linear}
The value functions $V_r \big( (b, q) \big)$ and $V_p \big( (b, q) \big)$ and Q-functions $Q_r \big( (b, q), a \big)$ and $Q_p \big( (b, q), a \big)$  are convex and piecewise linear in $b$.
\end{lemma}

\begin{proof}
It suffices to prove the statement for each $q \in Q$, which follows directly from~\cite{smallwood_OptimalControlPartially_1973}.
\end{proof}

Based on \cref{lem:piecewise linear}, we can represent the Q-function $Q_r$ (and the same for $Q_p$) on the product space by 
\begin{align} 
& Q_r \big( (b, q), a \big) = \max_{\theta \in \Theta_{q, a}} \sum_{s \in S} \theta(s) b(s) \label{eq:piecewise linear}
\\ & V_r \big( (b, q) \big) = \max_{\theta \in \Theta_{q}} \sum_{s \in S} \theta(s) b(s)
\end{align}
where each $\Theta_{q}$ and $\Theta_{q, a}$ (for $q \in Q$ and $a \in A$) is a finite set of $|S|$-dimensional vectors, which define hyperplanes on the belief space. Following~\eqref{eq:BellmanR_0}, we can take
\begin{equation} \label{eq:theta q}
\Theta_{q} = \cup_{a \in A } \Theta_{q, a}.
\end{equation}

Our goal is to identify the set $\Theta_{q, a}$ (thus $\Theta_{q}$). By plugging~\eqref{eq:piecewise linear}-\eqref{eq:theta q} into the the Bellman equation~\eqref{eq:BellmanR} and using \cref{def:product}, we have  
\begin{align} \label{eq:qr update}
Q_r \big( (b, q), a \big) = & \sum_{s \in S} r(s) b(s) + \gamma \int_{b' \in \dist(S)}
\notag \\ & T_\mathcal{D} (b, a, b') \max_{\theta \in \Theta_{q'}} \sum_{s \in S} \theta(s) b'(s) \mathrm{d} b'
\end{align}
where $q' = \delta(q, L(b))$. The updated $Q_r$ by~\eqref{eq:qr update} is again piecewise linear and convex in $b$.

The update rule~\eqref{eq:qr update} is not directly usable since the belief transition probabilities $T_\mathcal{D}$ are unknown. For learning, we can replace $T_\mathcal{D}$ with its empirical estimation. Suppose starting from the belief state $b$, we take the action $a$ repeatedly for $n$ times, and derive $n$ observations $o_1, \ldots, o_n$ and correspondingly $n$ updated beliefs $b'_1, \ldots, b'_n$ by~\eqref{eq:belief update}. Then, the empirical estimation of $T_\mathcal{D}$ is given By
\begin{equation} \label{eq:empirical}
\hat{T}_\mathcal{D} = \frac{1}{n} I (b, a, b'_n)
\end{equation}
where $I$ is the indicator function.

By applying~\eqref{eq:empirical} to~\eqref{eq:qr update}, we derive
\begin{align} \label{eq:qr update empirical}
Q_r \big( (b, q), a \big) \gets & \sum_{s \in S} r(s) b(s) +  \frac{\gamma}{n} \sum_{i=1}^n \max_{\theta \in \Theta_{q'}} \sum_{s \in S} \theta(s) b_i'(s)
\end{align}
where $q' = \delta(q, L(b))$. Accordingly, we update $\Theta_{q, a}$ by
\begin{align} \label{eq:theta update}
\Theta_{q, a} \gets & \Theta_{q, a} \cup \Big\{ r  + \frac{\gamma}{n} \sum_{i=1}^n \argmax_{\theta \in \Theta_{q'}} \sum_{s \in S} \theta(s) b_i'(s)
\Big\}
\end{align}
The updated $\Theta_{q, a}$ may contain duplicated vectors dominated by others in taking the maximum in~\eqref{eq:piecewise linear}. We can prune them by linear programming.

\begin{theorem} \label{thm:2}
By iteratively updating $\Theta_{q, a}$ by~\eqref{eq:theta update} for all $q \in Q$ and pruning, the Q-function $Q_r$ and value function $V_r$ converges. The same holds for $Q_p$ and $V_p$.
\end{theorem}

\begin{proof}
First, as the number of samples $n \to \infty$, we have $\hat{T}_\mathcal{D} \to T_\mathcal{D}$. Then by repeatedly taking all $q \in Q$, the value functions $V_r \big( (\cdot, q) \big)$ (or $V_p$) converges for each $q$ by~\cite{smallwood_OptimalControlPartially_1973}.
\end{proof}

\begin{remark} \label{rem:1}
We can reduce the pruning complexity by choosing a finite set of witness beliefs $W \subseteq \dist(S)$. We keep $\theta \in \Theta_{q, a}$ if and only if it defines the value functions at some $w \in W$, i.e., $\theta = \argmax_{\theta \in \Theta_{q, a}} \sum_{s \in S} \theta(s) w(s)$. The point-based pruning method is computationally simpler at the cost of introducing a bounded error related to the density of $W$ (see~\cite{pineau_PointbasedValueIteration_2003} for details).   
\end{remark}

\subsection{Learning Algorithm}

We now present a learning method to solve \cref{problem} when the transition probabilities of the POMDP $\mathcal{M}$ is unknown. Our approach simultaneously runs two learning algorithms on the product belief MDP $\mathcal{D}^\times$ to solve for the Q-functions $Q_r$ and $Q_p$ (and the value functions $V_r$ and $V_p$). For $Q_r$, we use Q-learning with $\epsilon$-greedy policy exploration. Meanwhile, we use the Q-learning method from~\cite{bozkurt_ControlSynthesisLinear_2020} to update $Q_p$. The Q-function $Q_p$ determines the maximal satisfaction probability of the iLTL constraint $\varphi$. Therefore, to ensure the absolute satisfaction of $\varphi$, only the actions from
\begin{equation} \label{eq:allowed}
A_{\text{safe}} \big( (b,q) \big) = \big\{ a \in A \mid Q_p \big( (b,q), a \big) = 1 \big\}
\end{equation}
are allowable on the product belief state $(b, q)$ for the learning of $Q_p$ (excluding those $\epsilon$-greedy policy explorations).

We implement our learning method in an off-policy fashion for generality. We keep track of all previously-sampled episodes
\begin{equation} \label{eq:record}
\Xi \gets \Xi \cup \{ \big( (b, q), a, (b', q') \big) \}
\end{equation}
and update the empirical transition probabilities of beliefs $\hat{T}_\mathcal{D}$ accordingly. The new samples are drawn by randomly choosing a product belief $(b, q)$ that has appeared in $\Xi$. The corresponding action is selected as described above. The overall learning method is presented by~\cref{alg:1}. Finally, we can derive the policy that solves \cref{problem} by the discussion in \cref{sub:product belief mdp} using the learned Q-functions $Q_r$ and $Q_p$ from \cref{alg:1}.

\begin{theorem}
For a given POMDP $\mathcal{M}$ and iLTL safety constraints $\varphi$, Algorithm~\ref{alg:1} converges.
\end{theorem}

\begin{proof}
By \cref{thm:2}, both $Q_r$ and $Q_p$ converges in \cref{alg:1}, thus the claim holds.
\end{proof}

\begin{remark} \label{rem:2}
In the pruning step of \cref{alg:1}, we may use the beliefs in $\Xi$ as the witness beliefs (as discussed in \cref{rem:1}), since those beliefs are the most important for learning. We will study this problem in future work.
\end{remark}

\begin{algorithm}[!t]
\caption{Reinforcement learning on product belief MDP}
\label{alg:1}
\begin{algorithmic}[1]

\State \textbf{Input} iLTL constraint $\varphi$, POMDP $\mathcal{M}$

\State Build Belief MDP $\mathcal{D}$ for POMDP $\mathcal{M}$

\State Build LDBA $\mathcal{A}$ for iLTL formula $\varphi$

\State Build product belief MDP $\mathcal{D}^\times = \mathcal{D} \times \mathcal{A}$

\State Set parameters $T \gg 1$ and $\epsilon \ll 1$ 

\State Initialize $\Theta_{q, a}$ for $Q_r$ and $\Theta_{q, a}'$ for $Q_p$ 

\State $(b, q) \gets (p_0, q_0)$

\While{not converge}
\State Get $A_{\text{safe}}$ from~\eqref{eq:allowed}

\State Choose $a \in A_{\text{safe}} \big( (b,q) \big)$ with probability $1 - \epsilon$ 

\State \quad or any other in $A$ with probability $\epsilon$

\State Take $a$ and observe $o$

\State Compute $b'$ by~\eqref{eq:belief update} and $q' \gets \delta(q, L(b))$ with~\eqref{eq:labels}

\State Update $\Xi$ by~\eqref{eq:record} and $\hat{T}_\mathcal{D}$ by~\eqref{eq:empirical} 

\State Update and prune $\Theta_{q, a}$ (and $\Theta_{q, a}'$) by~\eqref{eq:theta update}

\State Randomly pick new $(b, q)$ that has appeared in $\Xi$
\EndWhile

\State \textbf{Return} $\Theta_{q, a}$ and $\Theta_{q, a}'$   

\end{algorithmic}
\end{algorithm}